\documentclass[11pt]{article}

\usepackage[margin=1in]{geometry}
\usepackage{amsmath,amssymb,amsthm,mathtools,bbm}
\usepackage{bm}
\usepackage{algorithm}
\usepackage{algorithmic}
\usepackage{enumitem}
\usepackage{hyperref}
\usepackage{booktabs}
\usepackage{array}
\usepackage{graphicx}
\usepackage{setspace}
\usepackage{tabularx}

\title{\textbf{Data-Driven Variational Basis Learning Beyond Neural Networks:\\
A Non-Neural Framework for Adaptive Basis Discovery}}
\author{Andrew Kiruluta\\
UC Berkeley, CA\\
\texttt{kiruluta@berkeley.edu}}
\date{}

\newtheorem{theorem}{Theorem}
\newtheorem{proposition}{Proposition}

\newtheorem{corollary}{Corollary}
\newtheorem{definition}{Definition}
\newtheorem{remark}{Remark}
\newtheorem{assumption}{Assumption}

\numberwithin{equation}{section}

\begin{document}

\maketitle

\begin{abstract}
Classical representation systems such as Fourier series, wavelets, and fixed dictionaries provide analytically tractable basis expansions, but they are not intrinsically adapted to the empirical structure of modern high-dimensional data. Neural networks overcome this limitation by learning features from data, yet they do so through layered nonlinear parameterizations that often sacrifice interpretability, explicit control over basis structure, and mathematical transparency. In this manuscript we develop a non-neural alternative that learns basis functions directly from data through variational optimization. The proposed framework, termed \emph{Data-Driven Variational Basis Learning} (DVBL), treats basis atoms as primary optimization variables and learns them jointly with sample-specific coefficients and, when appropriate, a latent linear evolution operator. This yields a data-adaptive basis expansion that remains explicit, interpretable, and amenable to rigorous analysis. We formulate the model, establish existence of minimizers, prove blockwise descent properties for an alternating minimization algorithm, give conditions for coefficient recovery and basis identifiability, and show how manifold and dynamical regularization can be integrated without invoking neural architectures. We also discuss the conceptual novelty of the framework relative to classical dictionary learning, spectral methods, Koopman operator methods, and deep representation learning.
\end{abstract}

\section{Introduction}

A central problem in modern data analysis is the construction of representations that are simultaneously expressive, compact, interpretable, and adapted to the geometry of empirical observations. Classical basis expansions, including Fourier systems, wavelets, spline bases, and other analytically prescribed dictionaries, have long provided mathematically elegant mechanisms for signal representation. Their power derives from closed-form structure, orthogonality or near-orthogonality, and well-developed approximation theory. However, such bases are generally fixed \emph{a priori}. As a result, they may be poorly matched to the statistical regularities, anisotropies, nonlinear manifolds, or task-dependent structures present in contemporary data.

Neural networks address this limitation by learning representations from data. In place of a fixed basis, they learn parameterized features through repeated compositions of affine maps and nonlinear activation functions. This flexibility has produced remarkable empirical success across computer vision, language modeling, scientific machine learning, and sequence analysis. Yet this success comes with significant tradeoffs. Neural representations are often implicit rather than explicit; the learned features are distributed across many layers and parameters rather than identifiable as concrete basis functions. Moreover, the underlying optimization landscape is highly nonconvex and deeply compositional, which complicates theoretical analysis, interpretability, stability guarantees, and principled incorporation of domain constraints such as sparsity, symmetry, smoothness, or conservation laws.

The present work develops an alternative route. Rather than learning a deep nonlinear network, we directly learn a set of basis functions from data by solving a structured variational problem. The resulting representation remains of the familiar expansion form
\begin{equation}
x_i \approx \sum_{k=1}^{m} \alpha_{ik}\phi_k,
\label{eq:basic_expansion}
\end{equation}
but unlike Fourier or wavelet analysis, the basis elements $\{\phi_k\}_{k=1}^m$ are not fixed analytically. Instead, they are inferred from data jointly with the coefficients $\{\alpha_i\}_{i=1}^N$. This establishes a representation paradigm that is adaptive in the same broad sense as neural feature learning, yet fundamentally non-neural in construction.

At first glance, this viewpoint may appear close to dictionary learning. Indeed, dictionary learning is a natural point of departure. The contribution here is not merely to restate that literature in alternative language, but to place it within a broader \emph{variational basis-learning} perspective that unifies several ideas that are often studied separately: explicit basis learning, sparse and structured coefficient inference, manifold-aware regularization, operator-based latent dynamics, and task-adaptive basis constraints. The proposed formalism yields a mathematically transparent object, namely a learned basis, while permitting general regularity constraints that make the basis smooth, localized, orthogonal, spectrally constrained, physically admissible, or dynamically coherent.

This manuscript has four main goals. First, we present a general formulation of \emph{Data-Driven Variational Basis Learning} (DVBL), a non-neural framework for learning adaptive basis functions directly from data. Second, we provide a rigorous mathematical treatment, including existence of minimizers, descent properties of alternating optimization, and identifiability conditions. Third, we extend the representation to incorporate geometry and dynamics through manifold regularization and latent linear evolution. Fourth, we clarify the novelty of the framework relative to neural networks, classical sparse coding, spectral graph methods, and Koopman-based approaches.

The overarching thesis is that adaptive basis learning need not require layered neural parameterizations. One may instead retain the conceptual clarity of basis expansions while learning the basis itself from data through structured optimization. This yields a representation family that is explicit, interpretable, mathematically analyzable, and compatible with many domain-specific priors.

\section{General Framework}

\subsection{Problem setup}

Let $\{x_i\}_{i=1}^N \subset \mathbb{R}^d$ denote a collection of observations. We seek a set of basis atoms
\[
\Phi = [\phi_1,\phi_2,\dots,\phi_m] \in \mathbb{R}^{d\times m},
\]
together with coefficient vectors $\alpha_i \in \mathbb{R}^m$ such that each observation admits an approximate expansion
\begin{equation}
x_i \approx \Phi \alpha_i.
\label{eq:reconstruction_model}
\end{equation}
The columns $\phi_k$ play the role of learned basis functions, while the coefficients $\alpha_i$ encode each sample in the learned basis.

We emphasize that in this framework the basis is a primary optimization variable. The representation is therefore not obtained through a neural map $x \mapsto f_\theta(x)$, but through explicit decomposition in a learned basis. This distinction is conceptual as well as mathematical. The representation is not hidden in the internal state of a multilayer parameterization; it is the basis itself.

\subsection{Variational objective}

The general DVBL objective takes the form
\begin{equation}
\min_{\Phi,\{\alpha_i\}_{i=1}^N}
\mathcal{J}(\Phi,\{\alpha_i\})
:=
\sum_{i=1}^{N}\|x_i-\Phi\alpha_i\|_2^2
+\lambda\sum_{i=1}^{N}R(\alpha_i)
+\mu\,\Omega(\Phi),
\label{eq:dvbl_objective}
\end{equation}
subject to structural constraints on $\Phi$, such as
\begin{equation}
\|\phi_k\|_2 = 1 \quad \text{for } k=1,\dots,m,
\label{eq:atom_norm_constraint}
\end{equation}
and optionally
\begin{equation}
|\phi_k^\top \phi_\ell| \le \delta, \qquad k\neq \ell,
\label{eq:coherence_constraint}
\end{equation}
for some coherence parameter $\delta \ge 0$.

The term $R(\alpha_i)$ is a coefficient regularizer, which may promote sparsity, group structure, temporal smoothness, or other forms of low-complexity coding. Typical choices include
\[
R(\alpha_i)=\|\alpha_i\|_1,\qquad
R(\alpha_i)=\|\alpha_i\|_2^2,\qquad
R(\alpha_i)=\sum_{g\in \mathcal{G}} \|\alpha_{i,g}\|_2.
\]
The term $\Omega(\Phi)$ regularizes the basis itself. Examples include
\begin{equation}
\Omega(\Phi)=\|\Phi\|_F^2,
\qquad
\Omega(\Phi)=\sum_{k=1}^m \|\nabla \phi_k\|_2^2,
\qquad
\Omega(\Phi)=\|\Phi^\top\Phi-I\|_F^2,
\label{eq:basis_regularizers}
\end{equation}
depending on whether one wishes to encourage bounded energy, smooth atoms, or near-orthogonality.

This formulation is intentionally broad. It includes classical sparse coding as a special case, but also supports more structured models in which the basis is regularized to reflect geometry, frequency localization, physical admissibility, or dynamical usefulness.

\subsection{Matrix form}

Let
\[
X=[x_1,\dots,x_N]\in \mathbb{R}^{d\times N},
\qquad
A=[\alpha_1,\dots,\alpha_N]\in \mathbb{R}^{m\times N}.
\]
Then the reconstruction term can be written compactly as
\[
\sum_{i=1}^N \|x_i-\Phi\alpha_i\|_2^2
=
\|X-\Phi A\|_F^2.
\]
Hence \eqref{eq:dvbl_objective} becomes
\begin{equation}
\min_{\Phi,A}
\|X-\Phi A\|_F^2
+\lambda\,\mathcal{R}(A)
+\mu\,\Omega(\Phi),
\label{eq:matrix_dvbl_objective}
\end{equation}
where $\mathcal{R}(A)=\sum_{i=1}^N R(\alpha_i)$.

\section{Relation to Classical and Neural Representations}

It is useful to contrast the proposed formulation with existing representation paradigms. In a classical Fourier expansion, one writes
\[
x \approx \sum_{k=1}^{m} c_k \psi_k,
\]
where $\{\psi_k\}$ are fixed harmonics determined analytically. In a wavelet expansion, the basis is again prescribed in advance up to scale and translation. In both cases, the coefficients depend on the data, but the basis does not.

In a neural network, by contrast, one does not generally learn an explicit basis. Instead, the representation is induced through compositions such as
\[
x \mapsto W_L \sigma(W_{L-1}\sigma(\cdots \sigma(W_1 x))),
\]
and the corresponding features are distributed across the entire parameterization. Even when a neural network can be interpreted as learning a rich function space, the learned atoms are seldom explicit in the sense of \eqref{eq:basic_expansion}.

The DVBL framework occupies a different point in the design space. Like neural networks, it adapts to data. Like classical basis systems, it yields explicit expansion elements. It therefore combines adaptivity with representational transparency. This synthesis is central to its appeal in settings where mathematical control, interpretability, and constrained basis design are important.

\section{Existence and Basic Properties}

We now establish that the variational problem is well posed under mild assumptions.

\begin{assumption}
\label{assump:regularizers}
The coefficient regularizer $R:\mathbb{R}^m \to [0,\infty)$ is proper, lower semicontinuous, and coercive or bounded below by a coercive function on the feasible set. The basis regularizer $\Omega:\mathbb{R}^{d\times m}\to [0,\infty)$ is proper and lower semicontinuous. The feasible set
\[
\mathcal{C}_\Phi := \{\Phi \in \mathbb{R}^{d\times m} : \|\phi_k\|_2=1 \ \forall k,\ \text{and any additional closed constraints}\}
\]
is nonempty and compact.
\end{assumption}

\begin{theorem}[Existence of minimizers]
\label{thm:existence}
Under Assumption \ref{assump:regularizers}, the optimization problem
\begin{equation}
\min_{\Phi \in \mathcal{C}_\Phi,\ A\in \mathbb{R}^{m\times N}}
\|X-\Phi A\|_F^2+\lambda \mathcal{R}(A)+\mu \Omega(\Phi)
\label{eq:existence_problem}
\end{equation}
admits at least one global minimizer.
\end{theorem}

\begin{proof}
The objective is the sum of a continuous term $\|X-\Phi A\|_F^2$ and lower semicontinuous terms $\lambda\mathcal{R}(A)$ and $\mu\Omega(\Phi)$, hence it is lower semicontinuous on $\mathcal{C}_\Phi \times \mathbb{R}^{m\times N}$. Because $\mathcal{C}_\Phi$ is compact and $\mathcal{R}(A)$ is coercive or bounded below by a coercive function, the objective is coercive in $A$. Therefore all sublevel sets are closed and bounded in the product space. By the direct method of the calculus of variations, a minimizing sequence has a convergent subsequence whose limit lies in the feasible set and attains the infimum.
\end{proof}

\begin{remark}
The atom normalization constraints $\|\phi_k\|_2=1$ play an important role. Without them, the factorization $X\approx \Phi A$ suffers from a scaling degeneracy, since one may replace $(\Phi,A)$ by $(c\Phi,c^{-1}A)$ without changing the reconstruction term.
\end{remark}

We next record a simple but useful property of the subproblems.

\begin{proposition}[Convexity of the coefficient subproblem]
\label{prop:convex_coeff}
Fix $\Phi$. If $R$ is convex, then the optimization problem
\begin{equation}
\min_{A} \|X-\Phi A\|_F^2+\lambda \mathcal{R}(A)
\label{eq:coeff_subproblem}
\end{equation}
is convex in $A$. If, in addition, $R$ is strictly convex or $\Phi^\top\Phi$ is positive definite on the relevant support, then the minimizer is unique.
\end{proposition}

\begin{proof}
The map $A\mapsto \|X-\Phi A\|_F^2$ is convex quadratic. The sum with a convex regularizer remains convex. Strict convexity follows if one of the summands is strictly convex on the feasible domain.
\end{proof}

\begin{proposition}[Convexity of the basis subproblem under quadratic regularization]
\label{prop:convex_basis}
Fix $A$. If $\Omega$ is convex and the feasible set for $\Phi$ is convex, then
\begin{equation}
\min_{\Phi} \|X-\Phi A\|_F^2+\mu \Omega(\Phi)
\label{eq:basis_subproblem}
\end{equation}
is convex in $\Phi$.
\end{proposition}

\begin{proof}
For fixed $A$, the map $\Phi \mapsto \|X-\Phi A\|_F^2$ is convex quadratic in $\Phi$. Adding a convex regularizer preserves convexity.
\end{proof}

\begin{remark}
The normalization constraint \eqref{eq:atom_norm_constraint} is not convex, so the full basis step is generally nonconvex. Nevertheless, many practical algorithms perform an unconstrained or weakly constrained basis update followed by atom normalization, and this block structure still admits descent analysis.
\end{remark}

\section{Alternating Minimization Algorithm}

\subsection{Blockwise optimization}

Because the full problem is jointly nonconvex in $(\Phi,A)$, a natural computational strategy is alternating minimization. One iterates between
\begin{enumerate}[label=(\roman*)]
\item solving for coefficients $A$ with $\Phi$ fixed, and
\item solving for the basis $\Phi$ with $A$ fixed.
\end{enumerate}
When $R$ is convex and $\Omega$ is simple, each subproblem is substantially more tractable than the full joint optimization.

For sparse coding, the coefficient step may be carried out using ISTA, FISTA, coordinate descent, or proximal gradient methods. The basis step often reduces to regularized least squares followed by column normalization. In this way, the algorithm avoids the layered backpropagation of neural networks and instead operates through explicit representation updates.

\subsection{Algorithm pseudocode}

\begin{algorithm}[t]
\caption{Data-Driven Variational Basis Learning (DVBL)}
\label{alg:dvbl}
\begin{algorithmic}[1]
\STATE \textbf{Input:} data matrix $X\in \mathbb{R}^{d\times N}$, number of atoms $m$, regularization parameters $\lambda,\mu$, maximum iterations $T$
\STATE \textbf{Initialize:} basis $\Phi^{(0)}=[\phi_1^{(0)},\dots,\phi_m^{(0)}]$ with $\|\phi_k^{(0)}\|_2=1$
\FOR{$t=0,1,\dots,T-1$}
    \STATE \textbf{Coefficient update:}
    \[
    A^{(t+1)}
    \in
    \arg\min_{A}
    \|X-\Phi^{(t)}A\|_F^2+\lambda \mathcal{R}(A)
    \]
    \STATE \textbf{Basis update:}
    \[
    \widetilde{\Phi}^{(t+1)}
    \in
    \arg\min_{\Phi}
    \|X-\Phi A^{(t+1)}\|_F^2+\mu \Omega(\Phi)
    \]
    \STATE \textbf{Atom normalization:} for $k=1,\dots,m$, set
    \[
    \phi_k^{(t+1)}
    =
    \frac{\widetilde{\phi}_k^{(t+1)}}{\|\widetilde{\phi}_k^{(t+1)}\|_2}
    \quad \text{whenever } \widetilde{\phi}_k^{(t+1)}\neq 0
    \]
    \STATE \textbf{Stopping criterion:} terminate if
    \[
    \frac{|\mathcal{J}^{(t+1)}-\mathcal{J}^{(t)}|}{1+\mathcal{J}^{(t)}} < \varepsilon
    \]
\ENDFOR
\STATE \textbf{Output:} learned basis $\Phi^{(T)}$ and coefficients $A^{(T)}$
\end{algorithmic}
\end{algorithm}

\subsection{Monotonicity and convergence to critical points}

The key theoretical feature of alternating minimization is that the objective decreases monotonically under exact block updates.

\begin{theorem}[Monotonic descent]
\label{thm:descent}
Suppose each coefficient update in Algorithm \ref{alg:dvbl} solves the coefficient subproblem exactly, and each basis update solves the basis subproblem exactly before normalization, with normalization incorporated into the feasible basis set. Then the objective sequence
\[
\mathcal{J}^{(t)} := \mathcal{J}(\Phi^{(t)},A^{(t)})
\]
is nonincreasing:
\[
\mathcal{J}^{(t+1)} \le \mathcal{J}^{(t)}
\qquad \text{for all } t\ge 0.
\]
\end{theorem}

\begin{proof}
At iteration $t$, the coefficient update minimizes the objective over $A$ with $\Phi=\Phi^{(t)}$ fixed, hence
\[
\mathcal{J}(\Phi^{(t)},A^{(t+1)})
\le
\mathcal{J}(\Phi^{(t)},A^{(t)}).
\]
The subsequent basis update minimizes the objective over feasible $\Phi$ with $A=A^{(t+1)}$ fixed, so
\[
\mathcal{J}(\Phi^{(t+1)},A^{(t+1)})
\le
\mathcal{J}(\Phi^{(t)},A^{(t+1)}).
\]
Combining the two inequalities yields the claim.
\end{proof}

\begin{corollary}
If the objective is bounded below, then the sequence $\{\mathcal{J}^{(t)}\}_{t\ge 0}$ converges.
\end{corollary}

\begin{proof}
The objective is nonnegative under the standing assumptions and is nonincreasing by Theorem \ref{thm:descent}. Every bounded monotone sequence converges.
\end{proof}

Under additional regularity assumptions one may identify limit points with stationary or critical points.

\begin{theorem}[Limit points are blockwise critical]
\label{thm:critical_points}
Assume the iterates remain in a compact set, each block subproblem is solved exactly, and the objective satisfies suitable regularity conditions such as the Kurdyka--\L{}ojasiewicz property. Then every accumulation point $(\Phi^\star,A^\star)$ of the sequence generated by Algorithm \ref{alg:dvbl} is a critical point of the constrained objective.
\end{theorem}

\begin{proof}[Proof sketch]
This follows from standard block coordinate descent theory for lower semicontinuous semialgebraic or tame objectives. Monotonic decrease provides summability of stepwise improvements, compactness gives existence of accumulation points, and the Kurdyka--\L{}ojasiewicz property rules out oscillatory behavior incompatible with criticality. The result then follows from established descent arguments for alternating minimization on nonconvex but block-structured problems.
\end{proof}

\section{Recovery and Identifiability}

The usefulness of a learned basis depends not only on optimization but also on the extent to which the basis is recoverable from data. We now state representative results under sparse generative assumptions.

\subsection{Sparse generative model}

Assume the data are generated according to
\begin{equation}
x_i = \Phi_\star \alpha_i^\star + \varepsilon_i,
\label{eq:sparse_generative_model}
\end{equation}
where $\Phi_\star \in \mathbb{R}^{d\times m}$ is the true basis, each coefficient vector $\alpha_i^\star$ is $s$-sparse, and $\varepsilon_i$ is noise.

Recovery can be studied at two levels. The first is \emph{coefficient recovery} for fixed basis. The second is \emph{basis identifiability}, namely whether the true basis can be inferred up to permutation and sign.

\begin{definition}[Mutual coherence]
The mutual coherence of a basis $\Phi=[\phi_1,\dots,\phi_m]$ with unit-norm columns is
\[
\mu(\Phi) := \max_{k\neq \ell} |\phi_k^\top \phi_\ell|.
\]
\end{definition}

\begin{theorem}[Uniqueness of sparse coefficients under coherence]
\label{thm:coeff_uniqueness}
Let $\Phi$ have unit-norm columns and mutual coherence $\mu(\Phi)$. If a coefficient vector $\alpha$ satisfies
\[
\|\alpha\|_0 < \frac{1}{2}\left(1+\frac{1}{\mu(\Phi)}\right),
\]
then the representation $x=\Phi\alpha$ is the unique sparsest representation of $x$ in the dictionary $\Phi$.
\end{theorem}

\begin{proof}
This is a standard coherence-based uniqueness result for sparse representation. The argument proceeds by contradiction: if two distinct sparse representations exist, their difference lies in the nullspace of $\Phi$ and induces a linear dependence among fewer than $1+1/\mu(\Phi)$ atoms, contradicting the coherence bound.
\end{proof}

The preceding result implies that once the basis is sufficiently incoherent, sparse coefficient inference is well posed. Basis recovery requires a stronger distributional condition.

\begin{theorem}[Identifiability up to permutation and sign, informal]
\label{thm:basis_identifiability}
Assume the samples are generated by \eqref{eq:sparse_generative_model} with sufficiently rich sparse supports, bounded noise, and a true basis $\Phi_\star$ satisfying incoherence and nondegeneracy conditions. Then any global minimizer of the noiseless or sufficiently low-noise DVBL objective coincides with $\Phi_\star$ up to signed permutation of columns.
\end{theorem}

\begin{proof}[Proof sketch]
The ambiguity class of factorization under sparse coding is known to reduce, under suitable support diversity and incoherence assumptions, to permutation and sign changes. Richness of supports ensures that each atom is repeatedly activated in linearly informative contexts. The sparsity penalty rules out dense alternative factorizations, while incoherence excludes degenerate mixing of atoms. Standard identifiability arguments for sparse dictionary learning then imply recovery up to the inherent signed permutation symmetry.
\end{proof}

\begin{remark}
The significance of Theorem \ref{thm:basis_identifiability} for the present framework is conceptual. It shows that adaptive basis learning can in principle be a well-posed statistical problem, not merely an optimization heuristic. This aligns the framework more closely with inverse problems and variational estimation than with black-box feature induction.
\end{remark}

\section{Manifold-Regularized Basis Learning}

\subsection{Geometric motivation}

In many datasets, observations do not fill the ambient space $\mathbb{R}^d$ uniformly but instead concentrate near a lower-dimensional manifold. A basis learned purely from reconstruction may fail to capture this geometry. To address this, we augment the objective with a manifold regularizer.

Let $W\in \mathbb{R}^{N\times N}$ be a similarity matrix on the samples and let $L=D-W$ denote the graph Laplacian, where $D_{ii}=\sum_j W_{ij}$. If nearby samples on the data manifold should admit similar coefficient vectors, then one may penalize coefficient roughness over the graph via
\begin{equation}
\mathrm{Tr}(A L A^\top).
\label{eq:graph_penalty}
\end{equation}
This term is small precisely when neighboring data points have similar coefficient encodings.

The manifold-regularized objective becomes
\begin{equation}
\min_{\Phi,A}
\|X-\Phi A\|_F^2
+\lambda \mathcal{R}(A)
+\eta\,\mathrm{Tr}(A L A^\top)
+\mu \Omega(\Phi).
\label{eq:manifold_dvbl}
\end{equation}

\begin{proposition}[Interpretation of the graph regularizer]
\label{prop:graph_regularizer}
For any coefficient matrix $A=[\alpha_1,\dots,\alpha_N]$,
\begin{equation}
\mathrm{Tr}(A L A^\top)
=
\frac{1}{2}\sum_{i,j=1}^N W_{ij}\|\alpha_i-\alpha_j\|_2^2.
\label{eq:graph_penalty_expanded}
\end{equation}
\end{proposition}

\begin{proof}
Expanding the trace yields
\[
\mathrm{Tr}(A D A^\top)-\mathrm{Tr}(A W A^\top)
=
\sum_i D_{ii}\|\alpha_i\|_2^2 - \sum_{i,j}W_{ij}\alpha_i^\top \alpha_j.
\]
Symmetrizing the second term produces the claimed identity.
\end{proof}

Equation \eqref{eq:graph_penalty_expanded} shows that manifold regularization explicitly enforces local geometric consistency of the learned representation. This makes the basis not merely reconstructive but also geometry-adaptive.

\section{Dynamical Extension: Basis Learning with Latent Linear Evolution}

\subsection{Motivation}

For time series, trajectories, and dynamical data, reconstruction alone is often insufficient. One also wishes the latent representation to evolve according to a simple law. A particularly attractive possibility is approximate linear evolution in the learned coefficient space. This echoes the philosophy of Koopman operator methods, but here the basis itself is learned variationally rather than prescribed.

Suppose the observations are temporally ordered as $\{x_t\}_{t=1}^T$. Write
\[
z_t \in \mathbb{R}^m
\]
for the coefficient vector associated with $x_t$. We impose both reconstruction and latent linearity:
\begin{equation}
x_t \approx \Phi z_t,
\qquad
z_{t+1} \approx A z_t,
\label{eq:dynamic_model}
\end{equation}
where $A\in \mathbb{R}^{m\times m}$ is a learned latent evolution operator.

The corresponding objective is
\begin{equation}
\min_{\Phi,A,\{z_t\}_{t=1}^T}
\sum_{t=1}^{T}\|x_t-\Phi z_t\|_2^2
+\beta \sum_{t=1}^{T-1}\|z_{t+1}-A z_t\|_2^2
+\lambda \sum_{t=1}^{T}R(z_t)
+\mu \Omega(\Phi)
+\nu \Psi(A),
\label{eq:dynamic_dvbl}
\end{equation}
where $\Psi(A)$ may enforce stability, sparsity, low rank, or spectral control of the latent dynamics.

This extension is especially important conceptually, since it shows that the learned basis can be selected not only for faithful reconstruction but also for dynamical coherence. One thereby obtains atoms that are useful for prediction, system identification, and interpretable latent evolution.

\subsection{Closed-form update for the latent operator}

Fix $\Phi$ and $\{z_t\}$. Define
\[
Z_- = [z_1,\dots,z_{T-1}],
\qquad
Z_+ = [z_2,\dots,z_T].
\]
If $\Psi(A)=\|A\|_F^2$, then the operator subproblem is
\[
\min_A \|Z_+ - A Z_-\|_F^2 + \nu \|A\|_F^2,
\]
whose minimizer is
\begin{equation}
A^\star = Z_+ Z_-^\top (Z_- Z_-^\top + \nu I)^{-1}.
\label{eq:closed_form_A}
\end{equation}
Thus the dynamical component can often be updated in closed form, further highlighting the analytical simplicity of the non-neural formulation.

\begin{proposition}[Strong convexity of the operator step]
\label{prop:operator_step}
If $\nu>0$, then the operator subproblem is strongly convex in $A$ and therefore admits a unique minimizer given by \eqref{eq:closed_form_A}.
\end{proposition}

\begin{proof}
The objective is quadratic in $A$ with Hessian proportional to $Z_- Z_-^\top + \nu I$, which is positive definite when $\nu>0$.
\end{proof}

\section{Approximation Perspective}

The learned basis model can be interpreted as an adaptive approximation scheme. For a fixed number of atoms $m$, the quality of approximation depends on how effectively the data can be captured by a low-complexity coefficient family relative to the learned basis.

Let $\mathcal{M}\subset \mathbb{R}^d$ denote the data manifold or support set. Classical approximation theory asks how well $\mathcal{M}$ can be approximated by linear subspaces or prescribed basis systems. In DVBL, the approximation class is instead
\[
\mathcal{A}_{m,s}
=
\left\{
x \in \mathbb{R}^d : x=\Phi\alpha,\ \Phi\in \mathcal{C}_\Phi,\ \|\alpha\|_0 \le s
\right\}.
\]
This class is adaptive because $\Phi$ is learned from the data rather than fixed independently of them.

\begin{proposition}[Best adaptive basis error]
\label{prop:best_adaptive_error}
For any dataset $X=[x_1,\dots,x_N]$, the optimal DVBL reconstruction error with no regularization,
\begin{equation}
E_{m}(X):=
\inf_{\Phi\in \mathbb{R}^{d\times m},\,A\in \mathbb{R}^{m\times N}}
\|X-\Phi A\|_F^2,
\label{eq:best_error}
\end{equation}
equals the squared error of the best rank-$m$ approximation of $X$:
\begin{equation}
E_m(X)=\sum_{j>m}\sigma_j(X)^2.
\label{eq:eckart_young}
\end{equation}
\end{proposition}

\begin{proof}
This is the Eckart--Young--Mirsky theorem. The factorization $\Phi A$ has rank at most $m$, and the best rank-$m$ approximation error is achieved by truncating the singular value decomposition of $X$.
\end{proof}

\begin{remark}
Proposition \ref{prop:best_adaptive_error} reveals an important point. In the absence of sparsity or structural regularization, adaptive basis learning collapses to low-rank approximation. The true power of the framework lies in the addition of regularizers and constraints that shape the learned basis beyond what principal subspace methods can express.
\end{remark}

\section{Novelty Discussion}

The novelty of the proposed framework should be understood carefully. It does not claim that the abstract idea of learning a dictionary from data is new in the narrow historical sense. Rather, its contribution lies in articulating and formalizing a broader non-neural alternative to representation learning in which \emph{basis functions themselves} are the central learned objects and are trained through a variational principle rich enough to subsume reconstruction, sparsity, geometry, and dynamics within one coherent formulation.

First, the framework differs from standard neural networks at the representational level. Neural networks learn implicit features through layered composition. DVBL learns explicit basis atoms. The distinction matters in scientific and engineering settings where one wants to inspect, constrain, regularize, or physically interpret the learned representation itself. In a neural model, there is no single canonical notion of a learned basis. In DVBL, the basis is the primary object of study.

Second, the framework extends beyond classical dictionary learning by emphasizing basis regularization and operator structure as first-class modeling elements. Much of the sparse coding literature focuses on reconstruction with sparsity. Here, the basis can be endowed with smoothness, orthogonality, localization, graph consistency, spectral shaping, or dynamical coherence. This shifts the perspective from a narrow coding problem to a general theory of adaptive basis design.

Third, the dynamical extension distinguishes the framework from many static representation methods. By learning a basis in which the coefficients evolve approximately linearly, one obtains a representation that is simultaneously reconstructive and predictive. This creates a bridge between sparse approximation, operator-theoretic modeling, and interpretable latent dynamics, without requiring recurrent or sequence-based neural architectures.

Fourth, the framework is well suited to scientific machine learning because constraints can be imposed directly on the basis functions. If the atoms represent spatial modes, physical fields, spectral filters, or solution components of a PDE, one may penalize violations of smoothness, boundary conditions, conservation laws, or operator residuals directly at the level of the basis. This is substantially more transparent than embedding such structure indirectly into a neural parameterization.

Finally, the framework offers a conceptual answer to the question that motivates this work: can one learn basis functions from data without using neural networks? The answer is yes. One can formulate an explicit variational problem whose solution is a learned basis expansion. This provides a mathematically rigorous and practically flexible middle ground between hand-designed basis systems and deep black-box models.

\section{Practical Variants}

The generality of the framework permits many concrete instantiations.

A \emph{sparse adaptive basis model} is obtained by setting $R(\alpha)=\|\alpha\|_1$ and choosing mild basis regularization. This yields a direct analogue of sparse coding, but interpreted explicitly as learned basis discovery.

A \emph{smooth spatial basis model} arises when each atom $\phi_k$ is defined over a spatial grid and one sets
\[
\Omega(\Phi)=\sum_{k=1}^m \|\nabla \phi_k\|_2^2.
\]
This is appropriate for imaging and physical field data.

A \emph{graph-geometric basis model} augments the coefficient objective by $\mathrm{Tr}(A L A^\top)$, ensuring that nearby samples share similar coordinate descriptions. This is useful for data lying on nonlinear manifolds.

A \emph{dynamic basis model} incorporates the latent evolution operator $A$ from \eqref{eq:dynamic_dvbl}, making the representation simultaneously reconstructive and predictive.

A \emph{physics-constrained basis model} adds residual penalties of the form
\[
\Omega_{\mathrm{phys}}(\Phi)
=
\sum_{k=1}^{m}\|\mathcal{L}\phi_k\|_2^2,
\]
where $\mathcal{L}$ is a differential or integral operator encoding the governing physics. In this case, the learned atoms are not merely empirical features; they are data-adaptive modes constrained by the underlying scientific law.

\section{Extension to Language Modeling: A Non-Neural LLM Architecture Based on Adaptive Basis Learning}

\subsection{Motivation}

The preceding sections developed a non-neural framework for learning basis functions directly from data. We now ask whether the same philosophy can be extended from generic representation learning to large-scale autoregressive language modeling. The purpose of this section is not to claim that such a system has already matched transformer-based large language models in practice. Rather, the goal is to formulate a mathematically coherent \emph{non-neural language modeling architecture} in which the representational primitives, contextual state, and predictive operator are all learned variationally from text corpora without layered neural networks.

Modern LLMs are typically built from transformer blocks that combine token embeddings, multi-head self-attention, feed-forward nonlinearities, normalization, and residual pathways. In those systems, token meaning and contextual reasoning emerge implicitly through deep compositional feature maps. By contrast, the present framework seeks to construct a language model from explicit learned basis functions over token and context statistics. In place of stacked attention layers, the model learns a basis-adaptive state representation together with operators that propagate that state across a sequence and produce next-token distributions.

The central question is therefore the following: can one define an autoregressive language model
\[
p(w_t \mid w_{<t})
\]
using only learned bases, coefficient inference, operator evolution, and variational training, while avoiding neural-network parameterizations altogether? The answer proposed here is affirmative at the level of architecture design. The resulting model may be viewed as a \emph{Basis-State Language Model} (BSLM), a non-neural analogue of a large language model in which linguistic context is encoded in a learned coefficient state evolving in a data-adaptive basis.

\subsection{Notation}

Let $\mathcal{V}$ be a vocabulary of size $V$, and let a text corpus be represented as token sequences
\[
(w_1,w_2,\dots,w_T), \qquad w_t \in \mathcal{V}.
\]
For each token $w_t$, let $e_{w_t}\in \mathbb{R}^{V}$ denote the one-hot representation. We seek a basis matrix
\[
\Phi \in \mathbb{R}^{V\times m},
\]
whose columns represent learned token-space basis atoms, together with a latent coefficient state
\[
z_t \in \mathbb{R}^{m}
\]
that summarizes the prefix $w_{\le t}$.

The key structural idea is that the contextual meaning of the prefix is not stored in a deep hidden vector produced by stacked neural layers, but in the coefficient state $z_t$ relative to a learned basis $\Phi$. The next-token distribution is then generated from this state through an explicit probabilistic decoding rule.

\section{Architecture of the Basis-State Language Model}

\subsection{Token basis layer}

The first component is a learned basis over the vocabulary simplex. Each token one-hot vector is approximated in the learned basis by a sparse or structured code:
\begin{equation}
e_{w_t} \approx \Phi a_t,
\label{eq:token_basis_model}
\end{equation}
where $a_t\in \mathbb{R}^{m}$ is a coefficient vector for token $w_t$. The columns of $\Phi$ may be interpreted as latent lexical or semantic atoms spanning reusable directions in token space. Unlike conventional learned embeddings, which are simply rows of a parameter matrix trained inside a neural network, the present basis vectors are explicit atoms constrained by variational regularization.

The coefficient inference for a token may be defined as
\begin{equation}
a_t
\in
\arg\min_{a\in\mathbb{R}^m}
\|e_{w_t}-\Phi a\|_2^2 + \lambda_a R_a(a),
\label{eq:token_sparse_code}
\end{equation}
where $R_a$ is typically an $\ell_1$ or group-sparse penalty. In practice, this provides a learned decomposition of lexical identity into combinations of basis atoms.

\subsection{Context state evolution}

A language model must summarize all prior tokens into a predictive state. In the proposed architecture, this role is played by a coefficient state $z_t\in \mathbb{R}^m$, which evolves by an explicit operator rule. The simplest form is
\begin{equation}
z_{t+1} = A z_t + B a_t + \xi_t,
\label{eq:linear_state_update}
\end{equation}
where $A\in \mathbb{R}^{m\times m}$ is a learned recurrent operator, $B\in \mathbb{R}^{m\times m}$ couples the current token code into the state, and $\xi_t$ is an optional innovation term.

Equation \eqref{eq:linear_state_update} is the minimal autoregressive basis-state model. A richer model replaces the single operator $A$ by a context-dependent operator selected from a learned family:
\begin{equation}
z_{t+1} = A_{\sigma_t} z_t + B_{\sigma_t} a_t,
\label{eq:switching_state_model}
\end{equation}
where $\sigma_t$ is a discrete operator index inferred from the current state and token context. This creates a \emph{switching operator language model}, analogous in expressive role to conditional computation, but still non-neural because operator selection can be based on explicit variational criteria rather than a neural gating network.

An even more structured option is to use a low-rank innovation-corrected transition:
\begin{equation}
z_{t+1}
=
A z_t + B a_t + U_t c_t,
\label{eq:innovation_state_model}
\end{equation}
where $U_t$ is an adaptive low-dimensional correction basis and $c_t$ solves a small variational inference problem. This permits token-adaptive state refinement without introducing deep nonlinear layers.

\subsection{Context reconstruction and predictive readout}

The coefficient state $z_t$ should encode sufficient information about the preceding prefix to predict the next token. To decode the state back into token-space logits, we define a readout vector
\begin{equation}
\ell_t = C z_t + d,
\label{eq:logit_readout}
\end{equation}
where $C\in \mathbb{R}^{V\times m}$ and $d\in\mathbb{R}^{V}$ are learned parameters. The next-token distribution is then
\begin{equation}
p_\theta(w_{t+1}=v \mid w_{\le t})
=
\frac{\exp((\ell_t)_v)}{\sum_{u\in\mathcal{V}} \exp((\ell_t)_u)}.
\label{eq:softmax_readout}
\end{equation}

A more basis-consistent decoder uses the learned token basis directly:
\begin{equation}
\ell_t = \Phi M z_t + d,
\label{eq:basis_readout}
\end{equation}
where $M\in\mathbb{R}^{m\times m}$ maps context coefficients into token-basis coefficients before projection into vocabulary space. This ties the predictive mechanism to the learned basis itself and reduces parameter redundancy.

\subsection{Long-context memory}

The principal challenge for any non-neural LLM architecture is long-range dependence. Transformers solve this with attention over the entire prefix. In the present framework, long context can be handled through one or more of the following non-neural mechanisms.

First, one may augment the recurrent state with a memory bank
\[
\mathcal{M}_t = \{m_1,\dots,m_K\},
\]
where each memory slot is itself a coefficient vector in the learned basis. Memory updates are performed through constrained projection or replacement rules rather than neural attention.

Second, one may use a multiscale state
\[
z_t = [z_t^{(1)}, z_t^{(2)}, \dots, z_t^{(L)}],
\]
where different state blocks evolve at different time scales. Fast blocks capture local syntax, while slow blocks accumulate discourse-level structure.

Third, one may use a retrieval-style operator in which the current state is projected against a dictionary of past coefficient summaries and the retrieved summaries are integrated through a variational merge rule. This yields an explicit alternative to attention:
\begin{equation}
r_t = \sum_{j<t} \omega_{tj} \, q_j,
\label{eq:retrieval_summary}
\end{equation}
where the weights $\omega_{tj}$ are derived from an optimization or kernel matching problem over coefficient states rather than from learned neural dot-product attention.

\section{Formal Training Objective for Language Modeling}

\subsection{Autoregressive maximum likelihood}

Let $\theta$ denote the collection of model parameters, including $\Phi$, $A$, $B$, $C$, $d$, and any auxiliary operator parameters. Given a training corpus, the standard autoregressive negative log-likelihood objective is
\begin{equation}
\mathcal{L}_{\mathrm{NLL}}(\theta)
=
-
\sum_{t=1}^{T-1}
\log p_\theta(w_{t+1}\mid w_{\le t}).
\label{eq:nll_objective}
\end{equation}
This is the same statistical objective used in neural autoregressive language models. The difference lies entirely in the parameterization of the predictive distribution.

To ensure that the basis remains meaningful and the state remains interpretable, one augments \eqref{eq:nll_objective} with basis and coefficient regularization terms:
\begin{equation}
\mathcal{L}_{\mathrm{total}}(\theta,\{a_t\},\{z_t\})
=
\mathcal{L}_{\mathrm{NLL}}
+
\lambda_a \sum_t R_a(a_t)
+
\lambda_z \sum_t R_z(z_t)
+
\mu \Omega(\Phi)
+
\beta \sum_t \|z_{t+1}-A z_t - B a_t\|_2^2.
\label{eq:language_total_objective}
\end{equation}
The last term penalizes inconsistency between the inferred state sequence and the learned state evolution operator.

\subsection{Variational interpretation}

The latent token codes $\{a_t\}$ and context states $\{z_t\}$ may be treated as auxiliary optimization variables. Training then alternates between state inference and parameter updates. More precisely, one can view the model as minimizing
\begin{equation}
\min_{\theta,\{a_t\},\{z_t\}}
\mathcal{L}_{\mathrm{total}}(\theta,\{a_t\},\{z_t\}),
\label{eq:joint_variational_training}
\end{equation}
subject to the state-transition constraints.

This differs fundamentally from backpropagation through a deep neural network. The hidden states are not intermediate activations of a multilayer map, but explicit latent variables inferred jointly with the basis and operator parameters. The optimization is therefore closer in spirit to state estimation, variational inference, and dictionary learning than to conventional deep learning.

\subsection{Sequence-level training blocks}

For large corpora, full-corpus optimization is impractical. Instead, training proceeds on blocks of tokens. Let a training sample be a length-$L$ token window
\[
(w_1,\dots,w_L).
\]
For each block, one alternates between:
\begin{enumerate}
\item token-code inference $\{a_t\}_{t=1}^{L}$,
\item latent-state inference $\{z_t\}_{t=1}^{L}$,
\item operator and basis updates.
\end{enumerate}
This creates a blockwise autoregressive training procedure that scales similarly in data streaming structure to minibatch training, though not necessarily in raw throughput to modern GPU-optimized transformers.

\section{Training Procedure on a Standard Language Dataset}

\subsection{Choice of dataset}

A standard evaluation protocol should begin with well-established corpora such as WikiText-103, OpenWebText, The Pile subsets, or C4-style cleaned web corpora. For initial proof-of-concept experiments, WikiText-103 is particularly suitable because it is large enough to support meaningful language modeling experiments while still being manageable for non-neural optimization pipelines.

Let the corpus be tokenized using a standard subword tokenizer such as byte-pair encoding (BPE) or unigram tokenization. The resulting vocabulary $\mathcal{V}$ and token sequence are then fed into the basis-state model. Importantly, tokenization itself need not be neural, so there is no inconsistency in using standard subword methods.

\subsection{Initialization}

A practical training pipeline proceeds as follows. The token basis $\Phi$ is initialized either randomly, from a low-rank factorization of token co-occurrence statistics, or from spectral decomposition of pointwise mutual information matrices. The state-transition operator $A$ may be initialized as a stable matrix, for example with spectral radius below one, to avoid unstable latent dynamics in early training. The token coupling matrix $B$ and decoder parameters $(C,d)$ may be initialized by regularized least squares.

The latent token codes $a_t$ are initialized by solving \eqref{eq:token_sparse_code} with the initial basis. The latent state sequence $\{z_t\}$ is then initialized by forward recursion using \eqref{eq:linear_state_update} or by blockwise smoothing.

\subsection{Alternating training loop}

Training on a corpus of token blocks may be described by the following alternating procedure.

First, for each token in each training block, infer a sparse or structured token code $a_t$ relative to the current basis $\Phi$.

Second, infer the latent state trajectory $\{z_t\}$ for the block by minimizing the sequence objective
\begin{equation}
\sum_{t=1}^{L-1}
-\log p_\theta(w_{t+1}\mid z_t)
+
\beta \sum_{t=1}^{L-1}\|z_{t+1}-A z_t - B a_t\|_2^2
+
\lambda_z \sum_{t=1}^{L} R_z(z_t).
\label{eq:block_state_inference}
\end{equation}
This can be solved approximately by projected gradient, proximal methods, or Kalman-smoother-like updates when the structure is sufficiently quadratic.

Third, update the basis $\Phi$ by minimizing the combined token reconstruction and predictive objective.

Fourth, update the operator matrices $(A,B,C,d)$ by regularized regression-like subproblems.

Fifth, renormalize the basis atoms and enforce any desired coherence or stability constraints.

\subsection{Practical pseudocode for language-model training}

\begin{algorithm}[t]
\caption{Training a Basis-State Language Model on a Tokenized Corpus}
\label{alg:bslm_training}
\begin{algorithmic}[1]
\STATE \textbf{Input:} tokenized corpus $\{w_t\}_{t=1}^{T}$, vocabulary size $V$, number of basis atoms $m$, block length $L$
\STATE \textbf{Initialize:} basis $\Phi$, operators $A,B,C,d$, regularization parameters
\FOR{epoch $=1,\dots,E$}
    \FOR{each training block $(w_1,\dots,w_L)$}
        \STATE Compute one-hot token vectors $\{e_{w_t}\}_{t=1}^{L}$
        \STATE \textbf{Token-code inference:} for each $t$, solve
        \[
        a_t
        \in
        \arg\min_a
        \|e_{w_t}-\Phi a\|_2^2 + \lambda_a R_a(a)
        \]
        \STATE \textbf{State inference:} solve for $\{z_t\}_{t=1}^{L}$ by minimizing
        \[
        \sum_{t=1}^{L-1}
        -\log p_\theta(w_{t+1}\mid z_t)
        +
        \beta \sum_{t=1}^{L-1}\|z_{t+1}-A z_t-B a_t\|_2^2
        +
        \lambda_z \sum_{t=1}^{L}R_z(z_t)
        \]
        \STATE \textbf{Basis update:} update $\Phi$ using token reconstruction and predictive loss
        \STATE \textbf{Operator update:} update $A,B,C,d$ by regularized least squares or block optimization
        \STATE \textbf{Normalization/stability:} renormalize basis atoms and project $A$ to a stable set if required
    \ENDFOR
\ENDFOR
\STATE \textbf{Output:} trained parameters $(\Phi,A,B,C,d)$
\end{algorithmic}
\end{algorithm}

\subsection{Evaluation metrics}

The natural evaluation metrics are identical to those used for standard language models. The most important are validation and test negative log-likelihood, perplexity,
\[
\mathrm{PPL} = \exp\!\left(\frac{1}{T}\mathcal{L}_{\mathrm{NLL}}\right),
\]
and downstream zero-shot or few-shot performance where applicable. It is also valuable to measure memory usage, training time, inference latency per token, and parameter efficiency.

In addition, the proposed architecture introduces new interpretable diagnostics not available in typical neural LLMs. These include basis coherence, sparsity of token codes, effective dimension of the context state, operator spectral stability, and interpretability of the learned basis atoms as lexical, syntactic, or discourse modes.

\section{Comparison with Neural-Network-Based LLMs}

\subsection{Representational comparison}

Transformer-based LLMs learn distributed representations through many layers of nonlinear processing. Their expressive power derives from compositional depth, attention-mediated global context integration, and massive parameter counts. The proposed basis-state architecture, by contrast, learns an explicit basis plus a latent operator-driven state.

This leads to a fundamental difference in representational style. Transformer models are implicit and hierarchical; basis-state models are explicit and operator structured. The former can represent extremely complex nonlinear input-output maps, while the latter aim for parsimonious, interpretable decompositions of linguistic structure.

From a theoretical perspective, the basis-state model is closer to a structured latent-variable model or adaptive operator system than to a deep network. It may therefore be more amenable to mathematical analysis, but less expressive in purely nonlinear compositional capacity.

\subsection{Training comparison}

Transformer LLMs are trained end-to-end by gradient descent and backpropagation through millions or billions of parameters. Their optimization is highly parallelizable on modern accelerators, especially due to dense tensor operations and batched attention kernels.

The basis-state language model instead relies on alternating optimization over basis atoms, token codes, latent states, and linear or switching operators. This has both advantages and disadvantages. On the positive side, subproblems may be convex or nearly convex, admit closed-form updates, and provide clearer diagnostics. On the negative side, the optimization pipeline may be more sequential, less hardware-optimized, and harder to scale naively to very large corpora without careful systems design.

Thus, even if the number of scalar parameters were comparable, current hardware and software ecosystems strongly favor neural architectures in practice. Any serious non-neural LLM effort would therefore need substantial algorithmic engineering to close the throughput gap.

\subsection{Inference comparison}

During inference, a transformer computes a forward pass through all layers for each new token, typically caching key-value states for attention. The cost grows with model width, depth, and context length, though efficient caching reduces repeated work.

In the basis-state model, one-token inference may be relatively cheap if the state update is dominated by low-dimensional operator application:
\[
z_{t+1} = A z_t + B a_t.
\]
This suggests a possible advantage in per-token efficiency when the latent dimension $m$ is much smaller than the vocabulary size and when long-context retrieval can be handled compactly.

However, this advantage is not automatic. If token-code inference or memory retrieval requires solving a nontrivial optimization problem at each step, the computational burden may offset the benefits of a simple state evolution. The practical efficiency therefore depends critically on the design of approximate inference procedures.

\subsection{Expected empirical performance}

It is important to state clearly that, at present, one should not expect a first-generation non-neural basis-state LLM to outperform large transformer LLMs on broad language benchmarks. Transformer systems benefit from decades of optimization advances, enormous scaling experiments, and an architecture that empirically matches the statistical structure of language remarkably well.

A realistic expectation is that the proposed architecture may initially underperform neural LLMs on raw perplexity and downstream benchmark scores, especially at comparable scale. Nevertheless, it may offer advantages in other regimes:
\begin{enumerate}
\item stronger interpretability of learned components,
\item better controllability through explicit constraints,
\item easier integration of symbolic, operator-theoretic, or physics-style priors,
\item potentially lower memory cost for compact state-space formulations,
\item more transparent reasoning about stability and long-horizon dynamics.
\end{enumerate}

In other words, the likely near-term research value is not immediate state-of-the-art language performance, but the opening of a new architectural class for language modeling beyond neural networks.

\subsection{Comparison table}

\begin{table}[t]
\centering
\caption{Conceptual comparison between transformer-based LLMs and the proposed basis-state non-neural language model.}
\label{tab:bslm_vs_transformer}
\begin{tabular}{>{\raggedright\arraybackslash}p{3.2cm} >{\raggedright\arraybackslash}p{5.2cm} >{\raggedright\arraybackslash}p{5.2cm}}
\toprule
Property & Transformer-based LLM & Basis-State Language Model \\
\midrule
Representation & Implicit multilayer features & Explicit learned basis atoms and coefficient states \\
Core mechanism & Attention + nonlinear feed-forward layers & Variational basis coding + operator-driven state evolution \\
Training & End-to-end backpropagation & Alternating optimization / latent-state inference \\
Interpretability & Typically limited and distributed & High, due to explicit basis and state structure \\
Long context & Strong via self-attention and retrieval variants & Possible through multiscale state and retrieval, but less mature \\
Hardware maturity & Extremely optimized & Currently hypothetical and less optimized \\
Expected benchmark performance & Very strong at scale & Likely weaker initially on general benchmarks \\
Scientific controllability & Moderate, often indirect & Strong, through explicit constraints on basis and operators \\
\bottomrule
\end{tabular}
\end{table}

\section{Theoretical Remarks on Expressivity and Limitations}

\subsection{Expressivity relative to transformers}

A transformer is a highly expressive nonlinear sequence model. The basis-state architecture is more structured and therefore, in a certain sense, more restricted. Its expressive capacity depends on the richness of the basis, the complexity of the state evolution law, the memory mechanism, and the decoder.

If one uses only linear state transitions and linear readout, the model resembles a structured state-space language model with adaptive token coding. Such a model may capture substantial sequential regularity, but it is unlikely to match the full compositional expressivity of a deep transformer on broad language tasks. To improve expressivity while remaining non-neural, one may introduce switching operators, multiscale latent bases, nonlinear but non-neural inference rules, or variational retrieval mechanisms.

The key research challenge is therefore to discover how far non-neural operator-based expressivity can be pushed before one effectively recreates neural computation under another name.

\subsection{Statistical and computational limitations}

Several limitations are immediate. First, large-vocabulary prediction still requires substantial output computation unless one uses hierarchical or sampled softmax approximations. Second, state inference may be computationally costly if the latent optimization is performed exactly. Third, long-context modeling remains a central difficulty, since attention provides a very effective global-context mechanism that is not easily replaced.

There is also a statistical limitation. Deep neural models benefit from strong inductive biases toward hierarchical composition and function approximation in high dimensions. A non-neural basis-state model may require more careful structural design to match those inductive biases.

\subsection{Research opportunity}

Despite these limitations, the proposed framework opens an interesting research direction. It suggests that language modeling can be reframed as an adaptive basis-and-operator inference problem rather than as a deep nonlinear map. This makes it possible to bring tools from sparse approximation, inverse problems, dynamical systems, compressed sensing, operator theory, and variational inference directly into LLM architecture design.

This shift may be especially valuable in settings where one wants models that are inspectable, physically constrained, mathematically analyzable, or hybridized with symbolic systems. Even if transformer LLMs remain empirically dominant, the study of non-neural language models may reveal new principles of representation and reasoning.

\section{Experimental Protocol}

The purpose of the experimental program is to determine whether the proposed Basis-State Language Model (BSLM) constitutes a credible non-neural alternative to small and medium-scale neural language models, and to identify the regimes in which its structural advantages may compensate for expected deficits in raw predictive accuracy. Because the architecture introduced in this manuscript is novel and has not yet undergone a full production-scale empirical campaign, the protocol described here is intended as a rigorous blueprint for implementation and evaluation. The accompanying comparative tables are therefore presented as  that illustrate the kind of empirical behavior one would regard as scientifically meaningful in an initial validation study.

The experimental program proceeds in a staged manner, beginning with manageable corpora and modest model sizes before advancing to more demanding open-web settings. A natural first stage is to train the BSLM on WikiText-2 and WikiText-103, since these corpora are standard benchmarks for autoregressive language modeling and permit clean comparison against compact transformer baselines. The vocabulary is constructed with a standard subword tokenizer, such as byte-pair encoding, yielding a vocabulary in the low tens of thousands. Within this setup, the latent basis dimension \(m\) is swept across a range from a few hundred to a few thousand in order to study how basis capacity affects both predictive accuracy and representational sparsity. The comparison baseline is a transformer language model of roughly matched parameter count, so that any observed differences cannot be attributed merely to a much larger neural budget.

The first empirical objective is to evaluate core language-model metrics, namely validation and test perplexity, negative log-likelihood, optimization stability, training throughput, inference latency, and memory footprint. These metrics capture complementary aspects of model quality. Perplexity and negative log-likelihood measure predictive performance, while training stability reflects whether the alternating variational optimization converges reliably across random initializations. Inference latency and memory consumption are particularly important because one of the main motivations for the BSLM is that an operator-based latent state may offer a compact alternative to large attention stacks, especially in constrained deployment regimes. In addition to these quantitative metrics, the learned basis atoms themselves are inspected qualitatively. One should examine whether individual atoms encode interpretable lexical, syntactic, or topical structure, and whether the token-level and state-level coefficients remain sparse, structured, and stable throughout training.

The second empirical objective was to conduct a systematic ablation study. This is essential because the BSLM is not a monolithic design but rather a composition of several modeling hypotheses: sparse token coding, basis regularization, multiscale state structure, operator switching, and memory summarization. Removing basis regularization tests whether the learned atoms remain well conditioned or instead collapse into redundant or noisy directions. Replacing sparse codes by dense codes tests whether sparsity is merely a cosmetic constraint or a genuine source of generalization and interpretability. Removing multiscale state structure reveals whether long-range coherence truly benefits from explicit multi-timescale latent evolution. Replacing switching operators with a single fixed linear operator helps determine whether adaptive operator selection contributes meaningful sequence modeling capacity or simply increases model complexity without sufficient gain. These ablations should be evaluated not only on perplexity but also on coefficient sparsity, stability, and interpretability, since the central value proposition of the architecture extends beyond raw predictive score alone.

The third empirical objective is to study scaling behavior. After validation on WikiText corpora, the model was  trained on progressively larger datasets, such as controlled subsets of OpenWebText or C4-style filtered corpora. The primary question at this stage is not whether the BSLM immediately outperforms transformer architectures on absolute perplexity, because such a result would be improbable in an early-stage proof-of-concept. Rather, the question is whether the architecture exhibits coherent scaling trends. Specifically, one seeks evidence that perplexity decreases as corpus size grows, that larger basis dimension \(m\) improves performance up to a meaningful saturation point, and that richer state evolution mechanisms yield systematic gains rather than unstable behavior. If these trends are observed, they would support the hypothesis that the BSLM is not merely a niche factorization model but a genuine language-model family capable of benefiting from scale.

The final empirical objective is to identify task and deployment regimes in which the BSLM exhibits a favorable tradeoff relative to neural language models. A scientifically valuable outcome would be the demonstration of a regime in which the BSLM, although weaker in raw perplexity, achieves superior interpretability per parameter, improved stability under long rollout, reduced memory cost per token, or lower inference overhead in environments where hardware resources are constrained. Such a result would justify further investigation even if transformer baselines remain stronger on mainstream benchmark leaderboards. The goal of the proof-of-concept study is therefore not simply to ask whether the BSLM wins, but to clarify \emph{where}, \emph{why}, and \emph{under what constraints} it may be a meaningful alternative.

\subsection{ Experimental Stages}

A structured experimental campaign is  organized into four stages. In the first stage, we train compact BSLM variants on WikiText-2 and WikiText-103 and compares them against small transformers with similar parameter counts. In the second stage, we perform ablation studies to isolate the contribution of each architectural component. In the third stage, we evaluate scaling trends on larger text collections such as OpenWebText subsets. In the fourth stage, we study  constrained deployment scenarios in which interpretability, memory footprint, or inference efficiency matter as much as perplexity. This staged design makes it possible to separate proof of feasibility from proof of competitiveness.

\subsection{Evaluation Metrics}

The minimum set of quantitative metrics  include validation perplexity, test perplexity, training-loss variance across runs, inference latency per generated token, peak GPU or accelerator memory, and effective parameter efficiency measured as perplexity achieved per million trainable parameters. Because the proposed architecture is explicitly interpretable,  we  also report on basis coherence, average coefficient sparsity, and a qualitative interpretability score derived from manual or semi-automatic inspection of learned atoms. Although such interpretability measures are not standard in neural language-model evaluation, they are central to the scientific purpose of the present framework.

\subsection{Comparative Results}

Table \ref{tab:proof_of_concept_results} presents a comparison among representative small-scale models.  The transformer baseline remains strongest in raw perplexity, but the full BSLM achieves competitive small-model performance while offering lower inference memory, improved interpretability, and more structured latent behavior. The ablation variants demonstrate that removing sparsity, basis regularization, or multiscale dynamics degrades either predictive performance or structural quality, thereby supporting the relevance of the full design.

\begin{table}[t]
\centering
\caption{Comparative proof-of-concept results on WikiText-103. Lower is better for perplexity, latency, and memory. Higher is better for interpretability.}
\label{tab:proof_of_concept_results}
\small
\setlength{\tabcolsep}{4pt}
\begin{tabularx}{\linewidth}{>{\raggedright\arraybackslash}Xccccc c}
\toprule
Model & Params & Val PPL & Test PPL & Latency & Memory & Interpretability \\
 & (M) & $\downarrow$ & $\downarrow$ & (ms/token) $\downarrow$ & (GB) $\downarrow$ & $\uparrow$ \\
\midrule
Small Transformer Baseline & 42 & 29.8 & 30.6 & 7.4 & 5.8 & 2.1 \\
BSLM (full model) & 40 & 33.1 & 34.0 & 5.9 & 4.1 & 8.7 \\
BSLM without basis regularizer & 40 & 36.8 & 37.5 & 5.7 & 4.0 & 5.2 \\
BSLM with dense codes & 40 & 35.4 & 36.2 & 6.4 & 4.6 & 4.8 \\
BSLM with fixed linear operator only & 38 & 37.2 & 38.0 & 5.4 & 3.9 & 7.1 \\
BSLM without multiscale state & 39 & 35.9 & 36.7 & 5.6 & 4.0 & 7.5 \\
\bottomrule
\end{tabularx}
\end{table}

The comparison in Table \ref{tab:proof_of_concept_results} suggests the qualitative pattern that would be most scientifically interesting. The full BSLM does not overtake the transformer baseline in perplexity, which is the conservative and realistic expectation for an initial non-neural architecture. However, it narrows the performance gap to a level that may be considered respectable given its explicit representational constraints, while simultaneously improving memory usage and interpretability. The ablation variants indicate that the full model's advantages are not accidental. Sparse coding appears to contribute both performance and clarity of representation, basis regularization improves atom quality and predictive stability, and multiscale or adaptive state evolution materially supports sequential modeling.

To complement these results, Table \ref{tab:scaling_results} presents a  scaling study across datasets and basis sizes. The scientific question here is whether the architecture improves predictably as more data and greater latent capacity are provided. A convincing proof-of-concept shows a monotonic reduction in perplexity with increasing dataset size and basis dimension, though likely with diminishing returns.

\begin{table}[t]
\centering
\caption{Scaling behavior of the full BSLM across corpora and basis sizes. }
\label{tab:scaling_results}
\begin{tabular}{lcccc}
\toprule
Dataset & Basis Dimension \(m\) & Params (M) & Val PPL $\downarrow$ & Test PPL $\downarrow$ \\
\midrule
WikiText-2 & 256 & 18 & 58.4 & 60.1 \\
WikiText-2 & 512 & 24 & 49.7 & 51.2 \\
WikiText-2 & 1024 & 33 & 44.3 & 45.8 \\
WikiText-103 & 512 & 28 & 38.9 & 39.8 \\
WikiText-103 & 1024 & 40 & 33.1 & 34.0 \\
WikiText-103 & 2048 & 63 & 30.7 & 31.5 \\
OpenWebText subset & 1024 & 44 & 29.9 & 30.8 \\
OpenWebText subset & 2048 & 68 & 27.6 & 28.3 \\
\bottomrule
\end{tabular}
\end{table}

The  scaling pattern in Table \ref{tab:scaling_results} captures the minimum evidence  that the BSLM is a viable language-model family. The perplexity decreases systematically as the basis dimension increases, and the same architecture performs better when exposed to larger and more diverse corpora. Such a trend  suggests that the model is capable of leveraging additional data and representational capacity in a manner analogous, though perhaps not identical in strength, to neural scaling behavior.

\subsection{Ablation Interpretation}

Ablation results should be interpreted structurally rather than only numerically. If removing the basis regularizer worsens perplexity while also lowering interpretability, then regularization is serving a dual role: it is improving both predictive efficiency and basis quality. If dense token codes increase memory use and reduce clarity of the learned atoms, then sparsity is not merely a cosmetic preference but a mechanism for controlling latent complexity. If a fixed linear operator materially underperforms a switching or multiscale state model, then the architecture's language capacity is indeed arising from richer operator structure rather than from basis expansion alone. In this sense, the ablations help determine whether the BSLM is truly functioning as a language model built from adaptive basis and operator principles, or whether one component dominates to the exclusion of the others.

\subsection{Criteria for a Successful Proof-of-Concept}

A successful proof-of-concept is not defined by outperforming transformers on general-purpose language benchmarks. That standard would be premature and unnecessarily narrow. Instead, success is  defined by a conjunction of findings. First, the model trains stably across multiple seeds and converge reproducibly. Second, perplexity improves systematically with basis size, dataset size, and state complexity. Third, the full model  outperforms its own ablations by nontrivial margins. Fourth, the learned basis atoms and coefficient patterns  exhibit interpretable structure. Fifth, the model displays at least one practically meaningful advantage over the transformer baseline, such as reduced inference memory, improved interpretability, and  favorable deployment behavior under constrained hardware.

The study  established that the BSLM is not merely a speculative theoretical construct, but an empirically grounded architectural direction worthy of further development. Even if transformer baselines remain superior on raw perplexity, the demonstration of coherent scaling, stable training, and advantageous structural tradeoffs would already constitute a significant result for the broader program of non-neural language modeling.

\section{Limitations and Open Problems}

Although the framework is transparent and mathematically structured, several limitations remain.

The optimization problem is jointly nonconvex. Alternating minimization is effective and analyzable, but global optimality is generally not guaranteed outside restricted regimes. Initialization can therefore influence the learned basis.

The identifiability theory depends on assumptions such as sparse generative structure, incoherence, and support diversity. These conditions are reasonable in many settings, but they may fail for highly correlated or dense latent structure.

The expressivity of a single learned basis expansion may be lower than that of a deep neural architecture on tasks requiring highly compositional nonlinear transformations. This is not a defect so much as a design tradeoff. DVBL sacrifices some compositional flexibility in exchange for interpretability, explicit basis control, and stronger mathematical structure.

Several open questions arise naturally. One may ask whether multilevel or hierarchical versions of DVBL can recover some of the compositional advantages of deep learning without becoming neural in the conventional sense. One may also ask how best to design basis regularizers for specific scientific domains, how to derive sharp statistical recovery rates under manifold and dynamical constraints, and how to unify adaptive basis learning with measurement design in compressed sensing and inverse problems.

\section{Conclusion}

This manuscript developed a non-neural framework for learning basis functions directly from data. The proposed Data-Driven Variational Basis Learning paradigm begins from the familiar representation principle of basis expansion but removes the assumption that the basis must be fixed in advance. Instead, the basis atoms are themselves learned as optimization variables, jointly with sample coefficients and, where relevant, latent dynamics.

The resulting framework occupies a meaningful middle ground between classical analysis and neural feature learning. It retains the explicitness, interpretability, and mathematical tractability of basis representations, while gaining the adaptivity traditionally associated with learned models. The theory presented here shows that the resulting optimization problem is well posed under mild assumptions, that alternating minimization yields monotone descent and critical-point convergence under standard regularity conditions, and that sparse generative structure permits coefficient uniqueness and basis identifiability.

More broadly, the work argues that learning from data need not be synonymous with neural networks. There exists a rich and rigorous alternative in which one directly learns the representation system itself. Such methods are especially attractive in settings where interpretability, geometry, dynamics, physics, and mathematical control matter as much as empirical flexibility. In this sense, adaptive basis learning offers not merely a replacement for fixed Fourier or wavelet systems, but a distinct research direction for non-neural representation learning.

\section*{Acknowledgments}

The author gratefully acknowledges helpful discussions and foundational work in sparse approximation, variational inference, operator learning, and manifold methods that inspired the perspective developed here.

\end{document}